%% file: paper.tex
\documentclass[format=acmsmall, review=false]{article}
\usepackage[final]{nips_2017}

\usepackage{amsfonts}
\usepackage{amsmath}
\usepackage{amsthm}
\usepackage{dsfont}
\usepackage{enumitem}
\usepackage{graphicx}
\usepackage{mathrsfs}
\usepackage{algorithm}
\usepackage{algorithmic}
\usepackage{nicefrac}
\usepackage{hyperref}

\DeclareMathOperator*{\E}{\mathbb{E}}

\newcommand{\C}{\mathcal{C}}
\newcommand{\cR}{\mathcal{R}}
\newcommand{\cS}{\mathcal{S}}
\newcommand{\cT}{\mathcal{T}}

\newcommand{\Rev}{\text{Rev}}
\renewcommand{\r}{{r}}
\newcommand{\eps}{\epsilon}
\newcommand{\Rset}{\mathbb{R}}

\newcommand{\cX}{\mathcal{X}}
\newcommand{\h}[1]{\widehat{#1}}
\newcommand{\ind}{\mathds{1}}
\newcommand{\Int}{\int_0^1}

\newcommand{\set}[1]{\{#1\}}
\newcommand{\sgn}{\text{sign}}

\newcommand{\xx}{\mathsf{x}}

\newcommand{\mat}[1]{\mathbf{#1}}
\newcommand{\bx}{\mat{x}}

\newcommand{\one}{\mat 1}
\renewcommand{\t}{\mat{t}}

\newcommand{\ignore}[1]{}

\newtheorem{lemma}{Lemma}
\newtheorem{proposition}{Proposition}

\newtheorem{corollary}{Corollary}
\newtheorem{theorem}{Theorem}

\makeatletter
\newtheorem*{rep@theorem}{\rep@title}
\newcommand{\newreptheorem}[2]{%
\newenvironment{rep#1}[1]{%
 \def\rep@title{#2 \ref{##1}}%
 \begin{rep@theorem}}%
 {\end{rep@theorem}}}
\makeatother

\newreptheorem{theorem}{Theorem}
\newreptheorem{lemma}{Lemma}
\newreptheorem{corollary}{Corollary}

\hypersetup{
  colorlinks,
  citecolor = blue,
  linkcolor = blue
}
\title{Revenue Optimization with Approximate Bid Predictions}
\author{Andr\'es Mu\~noz Medina \\
Google Research \\
76 9th Ave \\
New York, NY 10011  
\And
{\bf Sergei Vassilvitskii} \\
Google Research \\
76 9th Ave \\
New York, NY 10011}

\begin{document}

\maketitle

\begin{abstract}%
In the context of advertising auctions, finding good reserve prices is
a notoriously challenging learning problem. This is due to the heterogeneity
of ad opportunity types, and the non-convexity of the objective
function. In this work, we show how to reduce reserve price
optimization to the standard setting of prediction under squared loss,
a well understood problem in the learning community.  We further bound
the gap between the expected bid and revenue in terms of the average
loss of the predictor. This is the first result that formally relates
the revenue gained to the quality of a standard machine learned model.
\end{abstract}

\input{intro}

\input{prelims}

\input{surplus_bound}
\input{experiments}
\input{conclusion}


\bibliographystyle{plainnat}
\bibliography{var_bound}
\newpage
\input{appendix}

\end{document}

%% file: intro.tex
\section{Introduction}
A crucial task for revenue optimization in auctions is setting a good
reserve (or minimum) price. Set it too low, and the sale may yield
little revenue, set it too high and there may not be anyone willing to
buy the item. The celebrated work by~\citet{Myerson81} shows how to
optimally set reserves in second price auctions, provided the value
distribution of each bidder is known.

In practice there are two challenges that make this problem
significantly more complicated. First, the value distribution is never
known directly; rather, the auctioneer can only observe
samples drawn from it. Second, in the context of ad auctions, the
items for sale (impressions) are heterogeneous, and there are
literally trillions of different types of items being sold. It is
therefore likely that a specific type of item has never been observed
previously, and no information about its value is known.

A standard machine learning approach addressing the heterogeneity problem is to
parametrize each impression by a feature vector, with the underlying
assumption that bids observed from  auctions with similar features will be similar.
In online advertising.  these features encode, for instance, the ad size, whether it's mobile or desktop,  etc. 

The question is, then, how to use the features to set a good reserve
price for a particular ad opportunity. On the face of it, this sounds
like a standard machine learning question---given a set of features,
predict the value of the maximum bid.  The difficulty comes from the
shape of the loss function. Much of the machine learning literature is
concerned with optimizing well behaved loss functions, such as squared
loss, or hinge loss. The revenue function, on the other hand is
non-continuous and strongly non-concave, making a direct attack a
challenging proposition.

In this work we take a different approach and reduce the problem of
finding good reserve prices to a prediction problem under the squared
loss. In this way we can rely upon many widely available and scalable
algorithms developed to minimize this objective. We proceed by
using the predictor to define a judicious clustering of the data, and
then compute the empirically maximizing reserve price for each group.
Our reduction is simple and practical, and directly ties the revenue
gained by the algorithm to the prediction error.

\subsection{Related Work}

Optimizing revenue in auctions has been a rich area of study,
beginning with the seminal work of ~\cite{Myerson81} who introduced
optimal auction design. Follow up work by ~\cite{Chawla} and 
~\cite{HartlineR09}, among others, refined his results
to increasingly more complex settings, taking into account multiple
items, diverse demand functions, and weaker assumptions on the shape
of the value distributions.

Most of the classical literature on revenue optimization focuses on the
design of optimal auctions when the bidding distribution of buyers is known.
More recent work has considered the computational and information
theoretic challenges in learning optimal auctions from data. 
 A long line of work
\citep{ColeRoughgarden, DevanurHuang, DhangwatnotaiRY15,
  MorgensternRoughgarden, MorgensternCOLT} analyzes the \emph{sample
complexity} of designing optimal auctions. The main contribution of
this direction is to show that under fairly general bidding
scenarios, a near-optimal auction can be designed knowing only a
polynomial number of samples from bidders' valuations.  Other authors,
~\citep{PaesLeme, RoughgardenWang} have focused on the computational
complexity of finding optimal reserve prices from samples, showing
that even for simple mechanisms the problem is often NP-hard to solve
directly.

Another well studied approach to data-driven revenue optimization is
that of online learning. Here, auctions occur one at a time, and the
learning algorithm must compute prices  as a function of the history of
the algorithm. These algorithms generally make no distributional
assumptions and measure their performance in terms of regret: the
difference between the algorithm's performance and the
performance of the best fixed reserve price in hindsight.
\cite{Kleinberg} developed an online revenue optimization
algorithm for posted-price auctions that achieves low regret. Their
work was later extended to second-price auctions by \cite{CesaBianchi}. 

A natural approach in both of these settings is to attempt to
\emph{predict} an optimal reserve price, equivalently the highest bid
submitted by any of the buyers. While the problem of learning this
reserve price is well understood for the simplistic model of buyers
with i.i.d. valuations \citep{CesaBianchi, DevanurHuang, Kleinberg},
the problem becomes much more challenging in practice, when the
valuations of a buyer also depend on features associated with the ad
opportunity (for instance user demographics, and publisher
information).

This problem is not nearly as well understood as its
i.i.d.\ counterpart.  \cite{MohriMunoz} provide learning guarantees
and an algorithm based on DC programming to optimize revenue in
second-price auctions with reserve. The proposed algorithm, however,
does not easily scale to large auction data sets as each iteration
involves solving a convex optimization problem. A smoother version of
this algorithm is given by \citep{Rudolph}. However, being a highly
non-convex problem, neither algorithm provides a guarantee on the
revenue attainable by the algorithm's output.  \cite{DevanurHuang}
give sample complexity bounds on the design of optimal auctions with
side information. However, the authors consider only cases where this
side information is given by $\sigma \in [0,1]$. More importantly,
their proposed algorithm only works  under the unverifiable
assumption that the conditional distributions of bids given $\sigma$
satisfy stochastic dominance.
\ignore{Finally,~\citep{Cui} proposes
partitioning data into clusters and solving the i.i.d. problem for
each cluster. The crucial choice of a partition, however, is heuristic
and thus provides no guarantees on the achievable revenue.}

{\bf Our results.} We show that given a predictor of the bid with
squared loss of $\eta^2$, we can construct a reserve function $r$ that
extracts all but $g(\eta)$ revenue, for a simple increasing function $g$. (See
Theorem ~\ref{thm:main} for the exact statement.)  To the best of our
knowledge, this is the first result that ties the revenue one can
achieve directly to the quality of a standard prediction task. Our
algorithm for computing $r$ is scalable, practical, and efficient.

Along the way we show what kinds of distributions are amenable to
revenue optimization via reserve prices. We prove that when bids
are drawn i.i.d.\ from a distribution $F$, the ratio between the mean
bid and the revenue extracted with the optimum monopoly reserve scales
as $O( \log {\bf Var}( F))$ -- Theorem ~\ref{thm:multapprox}. This
result refines the $\log h$ bound derived by ~\cite{HartlineLogN}, and
formalizes the intuition that reserve prices are more successful for
low variance distributions.

%
%

%% file: prelims.tex
\vspace{-.1in}
\section{Setup}
\label{sec:setup}
\vspace{-.1in}
We consider a repeated posted price auction setup where every auction
is parametrized by a feature vector $x \in \cX$ and a bid $b \in
[0,1]$. Let $D$ be a distribution over $\cX \times [0,1]$. Let $h
\colon \cX \to [0,1]$, be a bid 
prediction function and denote by $\eta^2$ the {\em squared loss}
incurred by $h$:
 \begin{equation*}
  \E[(h(x) - b)^2] = \eta^2.
\end{equation*}
We assume $h$ is given, and make no assumption on the structure of $h$
or how it is obtained. Notice that while the existence of such $h$ is
not guaranteed for all values of $\eta$, using historical data one
could use one of multiple readily available regression algorithms to find
the \emph{best} hypothesis $h$. 

Let $\cS = \big((x_1, b_1), \ldots, (x_m, b_m)\big) \sim D$ be a set
of $m$ i.i.d. samples drawn from $D$ and denote by $\cS_\cX = (x_1,
\ldots, x_m)$ its projection on $\cX$. Given a price $p$ let $\Rev(p,
b) = p \ind_{b \geq p}$ denote the revenue obtained when the bidder
bids $b$. For a reserve price function $\r \colon \cX \to [0,1]$ we
let:
\begin{equation*}
  \cR(\r) = \E_{(x, b) \sim D}\big[\Rev(\r(x), b)\big] \quad \text{and} \quad 
\h \cR(\r) = \frac{1}{m} \sum_{(x, b) \in \cS} \Rev(r(x), b)
\end{equation*}
 denote the expected and empirical revenue of reserve price function
$\r$. 

We also let $B = \E[b]$, $\h B = \frac{1}{m} \sum_{i=1}^m b_i$ denote
the population and empirical mean bid, and $S(\r) = B - \cR(\r)$, $\h
S(\r) = \h B - \h \cR(r)$ denote the expected and empirical {\em separation} between 
bid values and the revenue. Notice that for a given reserve price function $\r$, $S(\r)$
corresponds to \emph{revenue left on the table}. Our goal is, given
$\cS$ and $h$, to find a function $\r$ that maximizes $\cR(\r)$ or
equivalently minimizes $S(\r)$. 

\subsection{Generalization Error}

Note that in our set up we are only given samples from the
distribution, $D$, but  aim to maximize the {\em expected} revenue.
Understanding the difference between the empirical performance of an
algorithm and its expected performance, also known as the {\em
  generalization error},  is a key tenet of learning theory.

At a high level, the generalization error is a function of the
training set size: larger training sets lead to smaller generalization
error; and the inherent complexity of the learning algorithm: simple
rules such as linear classifiers generalize better than more complex
ones. 

In this paper we characterize the complexity of a class $G$  of
functions by its growth function $\Pi$. The growth function corresponds to the maximum number of \emph{binary
  labelings} that can be obtained by $G$ over all possible samples  $\cS_\cX$. 
%
%
It is closely related
to the VC-dimension when $G$ takes values in $\set{0,1}$ and to the
pseudo-dimension \citep{MorgensternRoughgarden, Mohribook} when $G$
takes values in $\Rset$. 

We can give a bound on the generalization error associated with
minimizing the empirical separation over a class of functions $G$. The
following theorem is an adaptation of Theorem 1 of \citep{MohriMunoz}
to our particular setup.
\begin{theorem}
\label{th:mohrimunoz}
Let $\delta > 0$, with probability at least $1 - \delta$ over the choice of
the sample $\cS$ the following bound holds uniformly for $r \in G$
\begin{equation}
\label{eq:generalization}
S(\r) \leq \h S(\r) + 2 \sqrt{\frac{\log 1/\delta}{2 m}}
 + 4 \sqrt{\frac{2 \log(\Pi(G, m))}{m}}.
\end{equation}
\end{theorem}
Therefore, in order to minimize the expected separation $S(\r)$ it suffices
to minimize the empirical separation $\h S(\r)$ over a class of
functions $G$ whose growth function scales polynomially in $m$.

\section{Warmup}
\label{sec:warmup}
In order to better understand the problem at hand, we begin by introducing
a straightforward mechanism for transforming the hypothesis function $h$
to a reserve price function $\r$ with guarantees on its achievable revenue.
\begin{lemma}
\label{lemma:offset}
Let $\r \colon \cX \to [0,1]$ be defined by $\r(x) := \max(h(x) - \eta^{2/3}, 0)$.
The function $\r$ then satisfies   $ S(\r) \leq \eta^{1/2} + 2 \eta^{2/3}.$
\end{lemma}
The proof is a simple application of Jensen's and Markov's inequalities
and it is deferred to Appendix~\ref{sec:proofs}.

This surprisingly simple algorithm shows there are ways to obtain
revenue guarantees from a simple regressor. To the best of our
knowledge these is the first guarantee of its kind. The reader may
be curious about the choice of $\eta^{2/3}$ as the offset in our
reserve price function. We will show that the dependence on
$\eta^{2/3}$ is not a simple artifact of our analysis, but a cost
inherent to the problem of revenue optimization.

Moreover, observe that this simple algorithm fixes a static offset, and does not 
make a distinction between those parts of the feature space, where the
algorithm makes a low error, and those where the error is relatively
high. By contrast our proposed algorithm partitions the space
appropriately and calculates a different reserve for each
partition. More importantly we will provide a data dependent bound on
the performance of our algorithm that only in the worst case scenario
behaves like $\eta^{2/3}$. 

\ignore {In the remainder of the paper we will introduce an algorithm with data
dependent bounds on $S(\r)$. Moreover, we will remove the dependence on
$\eta^{1/2}$ from the bound allowing for a wider range of values of
$\eta$ for which the bound on $S(\r)$ is meaningful.}

\section{Results Overview}
\label{sec:overview}
In principle to maximize revenue we need to find a class of functions 
$G$ with small complexity, but that contains a function which
approximately minimizes the empirical separation.  The
challenge comes from the fact that the revenue function, $\Rev$, is not
continuous and highly non-concave---a small change in the price, $p$,
may lead to very large changes in revenue.
%
%
This is the main reason why simply using the predictor $h(x)$ as a
proxy for a reserve function is a poor choice, even if its average
error, $\eta^2$ is small. For example a function $h$, that is just as
likely to over-predict by $\eta$ as to under predict by $\eta$ will
have very small error, but lead to $0$ revenue in half the cases.

A solution on the other end of the spectrum would simply memorize the
optimum prices from the sample $\cS$, setting $r(x_i) = b_i$.  While
this leads to optimal empirical revenue, a function class $G$ containing r would satisfy $\Pi(G, m) = 2^m$, making the bound of Theorem~\ref{th:mohrimunoz} vacuous. 


In this work we introduce a family $G(h,k)$ of classes parameterized
by $k \in \mathbb{N}$. This family admits an approximate minimizer that can be
computed in polynomial time, has low generalization error, and
achieves provable guarantees to the overall revenue.

More precisely, we show that given $\cS$, and a hypothesis $h$ with
expected squared loss of $\eta^2$:
\begin{itemize}[noitemsep,topsep=0pt]
\item For every $k \geq 1$ there exists a set of functions $G(h,k)$
  such that $\Pi(G(h, k),m) = O(m^{2k})$.
\item For every $k \geq 1$, there is a polynomial time algorithm that
  outputs $\r_k \in G(h, k)$ such that in the worst case scenario  $\h S(\r_k)$ is bounded by $O(\frac{1}{k^{2/3}} +  \eta^{2/3} + \frac{1}{m^{1/6}})$.
\end{itemize}

Effectively, we
show how to transform any classifier $h$ with low squared loss,
$\eta^2$, to a reserve price predictor that recovers all but
$O(\eta^{2/3})$ revenue in expectation.

\subsection{Algorithm Description}
\label{sec:algorithmdesc}

In this section we give an overview of the algorithm that uses both
the predictor $h$ and the set of samples in $\cS$ to develop a pricing
function $\r$. Our approach has two steps. First we partition the set
of feasible prices, $0 \leq p \leq 1,$ into $k$ partitions, $C_1, C_2,
\ldots, C_k$.  The exact boundaries between partitions depend on the
samples $\cS$ and their predicted values, as given by $h$. For each
partition we find the price that maximizes the empirical revenue in
the partition. We let $\r(x)$ return the empirically optimum price in
the partition that contains $h(x)$.

For a more formal description, let $\cT_k$ be the set of
$k$-partitions of the interval $[0,1]$ that is:
\begin{equation*}
  \cT_k = \{ \t = (t_0, t_1, \ldots, t_{k-1}, t_k) \ | \ 0 = t_0 < \ldots < t_k = 1\}.
\end{equation*}

We define $G(h, k) = \{ x \mapsto \sum_{j=0}^{k-1} r_i \ind_{t_j \leq
  h(x) < t_{j+1}} \ | \ r_j \in [t_i, t_{j+1}] \; \text{and} \; \t \in
\cT_k\} $.  A function in $G(h, k)$ chooses $k$ level sets of $h$ and
$k$ reserve prices. Given $x$, price $r_j$ is chosen if $x$ falls on
the $j$-th level set.

It remains to define the function $r_k \in G(h,k)$.
 Given a partition vector $\t \in \cT_k$, let the partition
 $\C^h = \set{C^h_1, \ldots, C^h_k}$  of $\cX$ be given  by 
$C^h_j = \set{x \in \cX  | t_{j-1} < h(x) \leq t_j}$. Let
 $m_j = |\cS_\cX \cap C^h_j|$ be the number of elements 
 that fall into the $j$-th partition. 
 
 We  define the predicted mean  and
variance of each group $C_j^h$ as
\begin{equation*}
\mu^h_j = \frac{1}{m_j} \sum_{x_i \in C^h_j} h(x_i)
\qquad \text{and} \qquad 
(\sigma^h_j)^2 = \frac{1}{m_j} \sum_{x_i \in C^h_j} (h(x_i) - \mu_j)^2.
\end{equation*}
We are now ready to present algorithm RIC-$h$ for computing $r_k \in H_k$.

\begin{algorithm}{{\bf R}eserve {\bf I}nference from {\bf C}lusters}
\vspace{-0.2in}
\begin{algorithmic}
\STATE Compute $\t^h \in \cT_k$ that minimizes $ \frac{1}{m} \sum_{j=0}^{k-1} m_j \sigma^h_j$.
\STATE Let $\C^h = C^h_1, C^h_2, \ldots, C^h_k$ be the induced partitions. 
\STATE For each $j \in 1, \ldots, k$, set $r_j = \max_r r  \cdot | \{i | b_i \geq r \wedge x_i \in C^h_j\}|$.
\STATE Return $x \mapsto \sum_{j=0}^{k-1} r_j \ind_{h(x) \in C^h_j}$.
\end{algorithmic}
\label{alg:clustering}
\vspace{-0.25in}
\end{algorithm}

Our main theorem states that the separation of $r_k$ is bounded by the
cluster variance of $\C^h$. For a partition $\C=\{C_1, \ldots, C_k\}$
of $\cX$ let $\sigma_j$ denote the empirical variance of bids for
auctions in $C_j$. We define the weighted empirical variance by:
\begin{equation}
  \label{eq:varfun}
  \Phi(\C) \colon =   \sum_{j=1}^k \sqrt{\sum_{i,i' : x_i, x_{i'} \in C_k} (b_i - b_{i'})^2}
  = 2  \sum_{j=1}^k m_j \h \sigma_j
  \end{equation}

\begin{theorem}
\label{thm:main}
  Let $\delta > 0$ and let $\r_k$ denote the output of
  Algorithm~\ref{alg:clustering} then $\r_k \in G(h, k)$ and  with
  probability at least $1 - \delta$ over the samples $\cS$:
  \begin{equation*}
    \h S(\r_k) \leq (3 \h B)^{1/3} \Big(\frac{1}{2m}\Phi(\C^h)\Big)
    \leq (3 \h B)^{1/3}\Big(\frac{1}{2 k}
    + 2 \Big(\eta^2 + \sqrt{\frac{\log1/\delta}{2 m}} \Big)^{1/2}\Bigg)^{2/3}.
  \end{equation*}
\end{theorem}
Notice that our bound is data dependent and only in he worst case scenario
it behaves like $\eta^{2/3}$. In general it could be much smaller.

We also show that the complexity of $G(h, k)$ admits a favorable
bound. The proof is similar to that in
\citep{MorgensternRoughgarden}; we include it in
Appendix ~\ref{app:generalization} for completness.
\begin{theorem}
\label{THM:ALG-GENERALIZATION}
The growth function of the class $G(h,k)$ can be bounded as:
$  \Pi(G(h, k), m) \leq \frac{m^{2 k - 1}}{k^k}.$
\end{theorem}

We can combine these  results with Equation
\ref{eq:generalization} and an easy bound on $\h B$ in terms of $B$ to
conclude:
\begin{corollary}
\label{cor:main}
  Let $\delta > 0$ and let $\r_k$ denote the output of
  Algorithm~\ref{alg:clustering} then $\r_k \in G(h, k)$ and  with
  probability at least $1 - \delta$ over the samples $\cS$:
\[
S(\r_k) \leq (3 \h B)^{1/3} \Big(\frac{\Phi(\C^h)}{2m}\Big) + O\Big(\sqrt{\frac{k \log m}{m}}\Big)
\leq (12B\eta^2)^{1/3} + O\Big(\frac{1}{k^{2/3}} + \Big(\frac{\log 
      1/\delta}{2 m}\Big)^{1/6}\!\!\! + \!\sqrt{\frac{k \log m}{m}}\Big) .
\]
\end{corollary}
Since $B \in [0,1]$, this implies that when $k = \Theta(m^{3/7})$, the separation is bounded by
$2.28 \eta^{2/3}$ plus additional error factors that go to 0 with
the number of samples, $m$, as $\tilde{O}(m^{-\nicefrac{2}{7}}).$

%% file: surplus_bound.tex
\section{Bounding Separation}

In this section we prove the main bound motivating our algorithm. This
bound relates the variance of the bid distribution and the maximum 
revenue that can be extracted when a buyer's bids follow such
distribution. It formally shows what makes a distribution 
\emph{amenable} to revenue optimization.

To gain intuition for the kind of bound we are striving for, consider
a bid distribution $F$. If the variance of $F$ is $0$, that is $F$ is
a point mass at some value $v$, then setting a reserve price to $v$
leads to no separation. On the other hand, consider the equal revenue
distribution, with $F(x) = 1 - \nicefrac{1}{x}$. Here any reserve
price leads to revenue of $1$. However, the distribution has unbounded
expected bid and variance, so it is not too surprising that more
revenue cannot be extracted. We make this connection precise, showing
that after setting the optimal reserve price, the separation can be
bounded by a function of the variance of the distribution.

Given any bid distribution $F$ over $[0,1]$ we denote by $G(r) = 1 -
\lim_{r' \to r^-} F(r')$ the probability that a bid is greater than or
equal to $r$.
Finally, we will let $R = \max_r r G(r)$ denote the maximum revenue
achievable when facing a bidder whose bids are drawn from distribution
$F$. As before we denote by $B = \E_{b \sim F}[b]$ the mean bid and by $S =
B - R$ the expected separation of distribution $F$. 

\begin{theorem}
\label{THM:EXPBOUND}
Let $\sigma^2$ denote the variance of $F$. Then
 $ \sigma^2 \geq 2 R^2 e^{\frac{S}{R}}  - B^2 - R^2.$
\end{theorem}
The proof of this theorem is highly technical and we present it in Appendix~\ref{app:moments}.

\begin{corollary}
\label{coro:additive}
The following bound holds for any distribution F:
$  S \leq  (3 R)^{1/3} \sigma^{2/3} \leq (3 B)^{1/3} \sigma^{2/3} $
\end{corollary}

The proof of this corollary follows immediately by an application of
Taylor's theorem to the bound of Theorem~\ref{THM:EXPBOUND}. It is
also easy to show that this bound is tight  (see Appendix ~\ref{app:lowerbound}).

\subsection{Approximating Maximum Revenue} 
In their seminal work \cite{HartlineLogN} showed that when faced with
a bidder drawing values distribution $F$ on $[1, M]$ with mean $B$, an
auctioneer setting the optimum monopoly reserve would recover at least
$\Omega(B / \log M)$ revenue.  We show how to adapt the result of
Theorem~\ref{THM:EXPBOUND} to refine this approximation ratio as a
function of the variance of $F$. We defer the proof to Appendix~\ref{sec:proofs}.

\begin{theorem}
\label{thm:multapprox}
For any distribution $F$ with mean $B$ and variance $\sigma^2$, the maximum revenue with monopoly reserves, $R$, satisfies: 
$\frac{B}{R} \leq 4.78 + 2 \log \big (1 + \frac{\sigma^2}{B^2}\big)$
\end{theorem}

Note that since $\sigma^2 \leq M^2$ this always leads to a tighter bound on the revenue. 

\subsection{Partition of \texorpdfstring{$\cX$}{X}}
\label{sec:partition}

Corollary~\ref{coro:additive} suggests clustering points in such a
way that the variance of the bids in each cluster is minimized. Given a partition
$\C = \{C_1,\ldots,C_k\}$ of $\cX$ we denote by
$m_j = |\cS_\cX \cup C_j|$,
$\h B_j = \frac{1}{m_j} \sum_{i: x_i \in C_j} b_i$,
$\h \sigma_j^2 = \frac{1}{m_j} \sum_{i: x_i \in C_j} (b_i - \h B_j)^2$. Let also
$r_j =  \arg\!\max_{p > 0} p |\set{b_i > p | x_i \in C_j}|$ and
$\h R_j = r_j |\set{b_i > r_j | x_i \in C_j}|$. 
\begin{lemma}
\label{lemma:clustering}
Let $r(x) = \sum_{j = 1}^k  r_j \ind_{x \in C_j}$ then
 $ \h S(r) \leq \Big(3 \h B\Big)^{1/3}\Big(\frac{1}{m}\sum_{j=1}^k m_j
 \h \sigma_j\Big)^{2/3} = \Big(3 \h B \Big)^{1/3}\Big(\frac{1}{2 m}
 \Phi(\C) \Big) .$
\end{lemma}
\begin{proof}
  Let $\h S_j = \h B_j \!\!-\!\!\h R_j$, Corollary~\ref{coro:additive}
applied to the empirical bid distribution in $C_j$ yields $
 \h S_j \!\!\leq\!\! (3 \h B_j)^{1/3} \h \sigma_j^{2/3}$. Multiplying by
 $\frac{m_j}{m}$, summing over all clusters and using H\"older's
 inequality gives:
\begin{align*}
  \h S(r) = \frac{1}{m} \sum_{j=1}^k m_j S_j
  &\leq \frac{1}{m} \sum_{j=1}^k (3 \h B_j)^{1/3} \h \sigma_j^{2/3} m_j  
\leq \Big(\sum_{j=1}^k \frac{3 m_j}{m} \h B_j \Big)^{1/3} \Big(\sum_{j=1}^k
  \frac{m_j}{m} \h \sigma_j \Big)^{2/3}.
\end{align*}

\end{proof}

\section{Clustering Algorithm}
\label{sec:approximatecluster}
In view of Lemma~\ref{lemma:clustering} and since the quantity $\h B$
is fixed, we can find a function minimizing the expected separation by
finding a partition of $\cX$ that minimizes the weighted variance $\Phi(\C)$ defined
Section~\ref{sec:algorithmdesc}.
From the definition of $\Phi$, this problem resembles a traditional $k$-means
clustering problem with distance function $d(x_i, x_{i'}) = (b_i -
b_{i'})^2$. Thus, one could use one of several clustering algorithms
to solve it.  Nevertheless, in order to allocate a new point $x \in
\cX$ to a cluster, we would require access to the bid $b$
which at evaluation time is unknown. Instead, we show how to utilize the
predictions of $h$ to define an almost optimal clustering of $\cX$.

For any partition $\C = \{C_1, \ldots, C_k\}$ of $\cX$ define
\begin{equation*}
\Phi_h(\C) = \sum_{j=1}^k
  \sqrt{\sum_{i,i' : x_i, x_{i'} \in C_k} (h(x_i) - h(x_{i'}))^2}.
\end{equation*}
Notice that $\frac{1}{2 m} \Phi_h(C)$ is the function minimized by
Algorithm~\ref{alg:clustering}.  The following lemma, proved in
Appendix~\ref{sec:proofs}, bounds the cluster variance
achieved by clustering bids according to their predictions.
\begin{lemma}
  \label{lemma:approxcluster}
  Let $h$ be a function such that $\frac{1}{m} \sum_{i=1}^m (h(x_i) - b_i)^2 \leq
  \h \eta^2$, and let $\C^*$ denote the partition that minimizes
  $\Phi(\C)$. If $\C^h$ minimizes  $\Phi_h(\C)$ then
  $\Phi(\C^h) \leq \Phi(\C^*)  + 4 m \h \eta$. 
\end{lemma}

\begin{corollary}
  \label{coro:algorithmguarantee}
  Let $r_k$  be the output  of Algorithm~\ref{alg:clustering}. If
  $\frac{1}{m} \sum_{j=1}^m (h(x_i)- b_i)^2 \leq \h \eta^2$ then:
  \begin{equation}
    \label{eq:empirical_varbound}
    \h S(r_k) \leq (3 \h B)^{1/3} \Big(\frac{1}{2 m} \Phi(\C^h)\Big)^{2/3}
    \leq (3 \h B)^{1/3} Big(\frac{1}{2 m} \Phi(\C^*) + 2 \h
    \eta\Big)^{2/3}.
  \end{equation}
\end{corollary}

\begin{proof}
  It is easy to see that the elements $C^h_j$ of $\C^h$ are of the
  form $C_j = \set{x | t_j \leq h(x) \leq t_{j+1}}$ for $\t \in
  \cT_k$. Thus if $r_k$ is the hypothesis induced by the partition
  $\C^h$, then $r_k \in G(h, k)$.  The result now follows by
  definition of $\Phi$ and lemmas~\ref{lemma:clustering} and
  \ref{lemma:approxcluster}.
\end{proof}

The proof of Theorem~\ref{thm:main} is now straightforward. Define a
partition $\C$ by $x_i \in C_j$ if 
$b_i \in \big[\frac{j-1}{k}, \frac{j }{k}\big]$.  Since
$(b_i - b_{i'})^2 \leq \frac{1}{k^2}$ for $b_i, b_{i'} \in C_j$ we have
  \begin{equation}
    \label{eq:unifpartition}
    \Phi(\C) \leq \sum_{j=1}^k \sqrt{\frac{m_j^2}{k^2}} = \frac{m}{k}.
  \end{equation}
 Furthermore since $\E[(h(x) - b)^2] \leq \eta^2$, 
 Hoeffding's  inequality implies that with probability $1 - \delta$:
  \begin{equation}
    \label{eq:hoeffding}
    \frac{1}{m}\sum_{i=1}^m (h(x_i) -b_i)^2 \leq \Big(
    \eta^2 + \sqrt{\frac{\log1/\delta}{2m}}\Big).
  \end{equation}
  In view of inequalities~\eqref{eq:unifpartition} and
  \eqref{eq:hoeffding} as well as Corollary~\ref{coro:algorithmguarantee}
  we have:
  \begin{equation*}
    \h S(r_k) \leq (3 \h B)^{1/3}\Bigg(\frac{1}{2m}\Phi(\C)
    + 2 \Big(\eta^2 + \sqrt{\frac{\log1/\delta}{2 m}} \Big)^{1/2}\Bigg)^{2/3}
    \leq (3 \h B)^{1/3}\Big(\frac{1}{2 k}
    + 2 \Big(\eta^2 + \sqrt{\frac{\log1/\delta}{2 m}} \Big)^{1/2}\Bigg)^{2/3}
  \end{equation*}

This completes the proof of the main result. To implement the algorithm, note that 
the problem of minimizing $\Phi_h(C)$ reduces to finding a
partition $\t \in \cT_k$ such that the sum of the variances within the
partitions is minimized. It is clear that it suffices to
consider points $t_j$ in the set $\mathcal{B} = \set{h(x_1),
  \ldots, h(x_m)}$. With this observation, a simple dynamic program
leads to a polynomial time algorithm with an $O(km^2)$ running time
(see Appendix~\ref{sec:algorithm}).
%
%

%% file: experiments.tex
\section{Experiments}
\label{sec:exp}
We now compare the performance of our algorithm against the following baselines:

\begin{enumerate}[noitemsep, topsep=0pt]
\item The offset algorithm presented in Section~\ref{sec:warmup}, where instead of using the theoretical offset $\eta^{2/3}$ we find the optimal $t$ maximizing the empirical revenue $\sum_{i=1}^m \big(h(x_i) - t) \ind_{h(x_i) - t \leq b_i}$.
\item The DC algorithm introduced by \cite{MohriMunoz}, which represents the state of the art in learning a revenue optimal reserve price. 
\end{enumerate}

{\bf Synthetic data.} We begin by running experiments on synthetic data to demonstrate the regimes where each algorithm excels. We generate feature vectors $\bx_i \in \Rset^{10}$ with coordinates  sampled from a mixture of lognormal distributions with means $\mu_1 = 0$, $\mu_2 = 1$, variance $\sigma_1 = \sigma_2 = 0.5$ and mixture parameter $p = 0.5$. Let $\one \in \Rset^{d}$ denote the vector with entries set to $1$. Bids are generated according to two different scenarios:
\begin{itemize}[noitemsep, topsep=0pt]
\item [{\bf Linear}] Bids $b_i$  generated according to $b_i = \max(\bx_i^\top \one + \beta_i, 0)$ where $\beta_i$ is a Gaussian random variable with mean $0$, and standard deviation $ \sigma \in \{0.01, 0.1, 1.0, 2.0, 4.0\}$. 
\item [{\bf Bimodal}] Bids  $b_i$ generated according to the following rule: let $s_i = \max(\bx_i^\top \one + \beta_i, 0)$ if $s_i > 30$ then $b_i = 40 + \alpha_i$ otherwise $b_i = s_i$. Here $\alpha_i$ has the same distribution as $\beta_i$. 
\end{itemize} 
The linear scenario demonstrates what happens when we have a good estimate of the bids. The bimodal scenario models a buyer, which for the most part will bid as a continuous function of features but that is interested in a particular set of objects (for instance retargeting buyers in online advertisement) for which she is willing to pay a much higher price. 

For each experiment we generated a training dataset $\cS_{train}$, a holdout set $\cS_{holdout}$ and a test set $\cS_{test}$ each with 16,000 examples. The function $h$ used by RIC-$h$  and the offset algorithm is found by training a linear regressor over $\cS_{train}$. For efficiency, we ran RIC-$h$ algorithm on quantizations of predictions $h(x_i)$.  Quantized predictions belong to one of 1000  buckets over the interval $[0,50]$.

Finally, the choice of hyperparameters $\gamma $ for the Lipchitz loss and $k$ for the clustering algorithm was done by selecting the best performing parameter over the holdout set. Following the suggestions in \citep{MohriMunoz} we chose $\gamma \in \{0.001, 0.01, 0.1, 1.0\}$ and $k \in \{2, 4, \ldots, 24\}$. 

Figure~\ref{fig:linear}(a),(b) shows the average revenue of the three approaches across 20 replicas of the experiment as a function of the log of $\sigma$. Revenue is normalized so that the DC algorithm revenue is $1.0$ when $\sigma=0.01$.  The error bars at one standard deviation are indistinguishable in the plot. It is not surprising to see that in the linear scenario, the DC algorithm of \citep{MohriMunoz} and the offset algorithm outperform RIC-$h$ under low noise conditions. Both algorithms will  recover a solution close to the true weight vector $\one$. In this case the offset is minimal, thus recovering virtually all revenue. On the other hand, even if we set the optimal reserve price for every cluster, the inherent variance of each cluster makes us leave some revenue on the table. Nevertheless, notice that as the noise increases all three algorithms seem to achieve the same revenue. This is due to the fact that the variance in each cluster is comparable with the error in the prediction function $h$.

The results are reversed for the bimodal scenario where RIC-$h$ outperforms both algorithms under low noise. This is due to the fact that RIC-$h$ recovers virtually all revenue obtained from high bids while the offset and DC algorithms must set conservative prices to avoid losing revenue from lower bids.

\begin{figure}[t]
\centering
\begin{tabular}{ccc}
\includegraphics[scale=0.35]{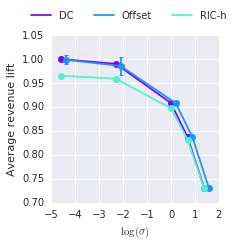}
& \includegraphics[scale=0.35]{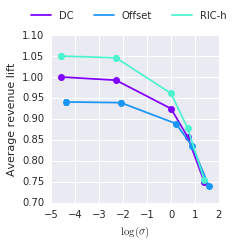}
& \includegraphics[scale=0.25]{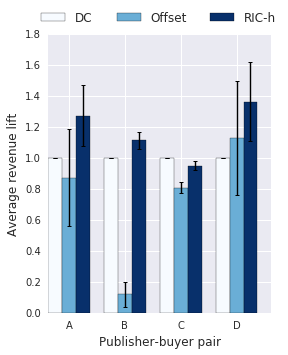}\\
(a) & (b) & (c)
\end{tabular}
\caption{(a) Mean revenue of the three algorithms on the linear scenario. (b) Mean revenue of the three algorithms on the bimodal scenario. (c) Mean revenue on auction data. }
\label{fig:linear}
\vspace{-0.20in}
\end{figure}

{\bf Auction data.} In practice, however, neither of the synthetic
regimes is fully representative of the bidding patterns. In order to
fully evaluate RIC-$h$, we collected auction bid data from AdExchange
for 4 different publisher-advertiser pairs. For each pair we sampled
100,000 examples with a set of discrete and continuous
features. The final feature vectors are in $\Rset^d$ for $d \in [100,
  200]$ depending on the publisher-buyer pair.  For each experiment,
we extract a random training sample of 20,0000 points as well as a
holdout and test sample. We repeated this experiment 20 times and
present the results on Figure~\ref{fig:linear} (c) where we have
normalized the data so that the performance of the DC algorithm is
always 1. The error bars represent one standard deviation from the
mean revenue lift.  Notice that our proposed algorithm achieves on
average up to $30\%$ improvement over the DC algorithm. Moreover, the
simple offset strategy never outperforms the clustering algorithm, and
in some cases achieves significantly less revenue.

%% file: conclusion.tex
\section{Conclusion}

We provided a simple, scalable reduction of the problem of revenue
optimization with side information to the well studied problem of
minimizing the squared loss. Our reduction provides the \emph{first}
polynomial time algoritm with a quantifiable bound on the achieved
revenue. In the analysis of our algorithm we also provided the first
variance dependent lower bound on the revenue attained by setting
optimal monopoly prices.  Finally, we provided extensive empirical
evidence of the advantages of RIC-$h$ over the current state of
theart. 

%% file: appendix.tex
\appendix
\section{Proof of Theorem~\ref{THM:EXPBOUND}}
\label{sec:prooftheorem}
Recall that given any bid distribution $F$ over $[0,1]$ we denote by $G(r) = 1 -
\lim_{r' \to r^-} F(r')$ the probability that a bid is greater than or
equal to $r$. Let $\xx(q) = \sup \{ x | G(x) \geq q\}$ denote
the pseudo-inverse of $G$. Notice in particular that when $G$ is
strictly decreasing then $\xx = G^{-1}$. When it is clear from context
we will refer to a distribution indistinctly by $F$, $G$ or $\xx$.

\label{app:moments}
We will use the following expressions for the expected bid and second
moment of a distribution.
\begin{lemma}
\label{LEM:MOMENTS}
The expected bid and second moments of any distribution $F$ are given respectively
by:
\begin{equation*}
 B = \Int \xx(q) dq  \quad \text{and}  \quad s^2 = \Int \xx(q)^2 dq.
\end{equation*}
\end{lemma}

\begin{proof}
We show the result only for the mean as the proof for
the second moment is similar. It is well known that for a positive
random variable, the mean can be expressed as:
\begin{equation*}
  B =  \int_0^\infty G(x) dx = \int_0^\infty \int_0^{G(x)}  dq dx =
  \int_D  dq dx.
\end{equation*} 
where $D = \{(x, q) \ | \  x > 0 \; \text{and} \; q \leq
G(x)\}$. Let $D' = \{(x, q) \ | \ 0 \leq q \leq 1 \; \text{and} \; x \leq
\xx(q)\}$. It is immediate that $D \subset D'$ as $q \leq G(x)$
implies by definition that $x \leq \xx(q)$. We can thus decompose the
above integral as:
\begin{equation*}
  \int_D  dq dx = \int_{D'}  dq dx  - \int_{D' - D}  dq dx.
\end{equation*}
The proof will be complete by showing that $D' - D$ has Lebesgue
measure $0$. Indeed, in that case the above expression reduces  to:
\begin{equation*}
  \int_{D'}  dq dx = \int_0^1 \int_0^{\xx(q)} dx dq = \int_0^1 \xx(q) dq.
\end{equation*}
Let us then characterize points $(x, q) \in D' -D$. Notice that if
$(x, q) \notin D$ then $G(x) < q$ but this again by definition implies
$x \geq \xx(q)$. If $(x, q)$ is also in $D'$ then we must have $x =
\xx(q)$. From which we conclude that  $\lim_{x' \to x^-} G(x') \geq q > G(x)$.
 Thus $(x,y ) \in D' - D$ implies that $x$ is a discontinuity of
$G$. Finally, since $G$ is decreasing there can be at most a countable
number of discontinuities and thus $D'  - D$ has measure 0. 
\end{proof}


In order to show the bound of Theorem~\ref{THM:EXPBOUND} holds, we
consider the following optimization problem over the space of square
integrable functions $L^2[0,1]$:
\begin{align}
  \min_{\xx \in L_2[0,1]}& \quad \frac{1}{2} \Int \xx^2(q)
                         dq \label{eq:optprob} \\
\text{s.t.} & \quad \Int \xx(q) = B \quad \text{and}
\quad  R \geq q \, \xx(q) \; \forall q. \nonumber
\end{align}
We show that the value of this optimization problem is greater than $
\frac{1}{2} \big( 2 R^2 e^{\frac{S}{R}} - R^2\big)$. Since 
any distribution $\xx(q)$ achieving revenue $R$ and separation $S$ is
feasible for \eqref{eq:optprob} it follows that it must satisfy 
$\sigma^2 = s^2 - B^2 \geq 2 R^2 e^{\frac{S}{R}} - B^2 - R^2$.

\begin{proposition}
\label{prop:dualproblem}
The objective value of \eqref{eq:optprob} is lower bounded by:
\begin{equation}
  \frac{B^2}{2} + \max_{v \in L_2[0,1] : v \geq 0}  -\frac{1}{2}
  \Big(\Int(q v(q))^2 - \Big(\Int q v(q) \Big)^2 \Big) + B
  \Int q v(q) - R\Int v(q). \label{eq:dualopt}
\end{equation}
\end{proposition}
\begin{proof}
For any $\lambda \in \Rset$ and $v \in L_2[0,1]$ define the Lagrangian
\begin{equation*}
L(x, \lambda, v) = \frac{1}{2}\Int \xx(q)^2 - \lambda \Int \xx(q) + \lambda B 
+ \Int v(q) (q\, \xx(q) - R).
\end{equation*}
It is immediate to see that optimization problem \eqref{eq:optprob} is
equivalent to 
\begin{equation*}
\min_{x\in L_2[0,1]} \max_{\lambda \in \Rset, v \geq 0} L(x, \lambda, v)
\geq \max_{\lambda \in \Rset, v \geq 0} \min_{x \in L_2[0,1]} L(x, \lambda, v).
\end{equation*}
By taking variational derivatives of the function $L$ with respect to
$\xx$ we see that the minimizing solution $\xx(q)$ satisfies:
\begin{equation*}
  \xx(q) = \lambda - q v(q).
\end{equation*}
Replacing this value in the function $L$ we see that problem
\eqref{eq:optprob} is lower bounded by:
\begin{equation*}
\max_{\lambda , v \geq 0}  \lambda B - R \Int v(q) - \frac{1}{2}
\Int (\lambda - q v(q))^2
\end{equation*}
We can solve for the unconstrained variable $\lambda$ to obtain
$\lambda = B + \Int q v(q)$. Replacing this value in the above
expression yields:
\begin{equation*}
  \max_{v \geq 0}  -R \Int v(q) + \frac{1}{2} \Big(\Int q v(q) + B\Big)^2  -
  \frac{1}{2} \Int (q v(q))^2.
\end{equation*}
Expanding the quadratic term yields:
\begin{equation*}
  \frac{B^2}{2} + \max_{v \geq 0 }  -\frac{1}{2} \Big(\Int(q v(q))^2
  - \Big(\Int q v(q)\Big)^2 \Big)  + B \Int q v(q) - R \Int v(q).
\end{equation*}
\end{proof}

To obtain a lower bound on \eqref{eq:dualopt} we simply need to
evaluate the objective function at a feasible function $v$. In
particular we let 
\begin{equation}
\label{eq:vq}
v(q) = \frac{R}{q}\Big[\frac{1}{s} - \frac{1}{q}\Big] \ind_{q > s}
\end{equation}
with $s = e^{-\frac{S}{R}}$. Notice that $v$ is clearly 
in $L_2[0,1]$ and $v \geq 0$. The choice of this function is highly
motivated by the solution to  the unconstrained version of
problem \eqref{eq:dualopt}.  

\begin{proposition}
\label{prop:explower}
The optimization problem 
\begin{equation}
 \max_{v \in L_2[0,1] : v \geq 0}  -\frac{1}{2}
  \Big(\Int(qv(q))^2 - \Big(\Int q v(q) \Big)^2 \Big) + B
  \Int q v(q) - R\Int v(q) \label{eq:maxopt}
  \end{equation}
is lower bounded by $\frac{1}{2} \Big( 2 R^2 e^{\frac{S}{R}} - R^2 - B^2 \Big)$. 
\end{proposition}

\begin{proof}
Let $v(q)$ be defined by \eqref{eq:vq}. Using the fact that $\log s =
-\frac{S}{R} = \frac{R - B}{R}$  we have the following equalities:
\begin{align}
  \Int q v(q) 
&= R \int_s^1 \frac{1}{s} - \frac{1}{q} = R\Big(\frac{1 -
    s}{s} + \log(s)\Big) = \frac{R}{s} - B \label{eq:intqvq} \\
  \Int v(q)
 &= R \Big(-\frac{\log s}{s} + \Big(1 - \frac{1}{s}\Big) \Big)
= \frac{R}{s} ( s - 1 - \log s) = \frac{ B + R ( s - 2)}{s}  \label{eq:intvq}.
\end{align}
In view of \eqref{eq:intqvq} and \eqref{eq:intvq} we have that for all $q \geq s$
\begin{equation*}
  q^2 v (q) - q \int_s^1 q v(q) = \frac{R q}{s} - R - \frac{R q}{s} +
  B q = B q - R
\end{equation*}
Multiplying the above equality by $v(q)$ and integrating we see the
objective function of \eqref{eq:maxopt} evaluated at $v(q)$ is given
by 
\begin{equation*}
\frac{1}{2} \Big( B \Int q v(q) - R \Int v(q)\Big).  
\end{equation*}
Replacing \eqref{eq:intqvq} and \eqref{eq:intvq} on the
expression above we obtain:
\begin{align*}
  \frac{1}{2} \Big( \frac{B R}{s} - B^2 - \frac{R B + R^2 (s -
  2)}{s}\Big) 
&= \frac{1}{2}\Big( \frac{2R^2}{s} - R^2  - B^2\Big) \\
&= \frac{1}{2} \Big( 2 R^2 e^{\frac{S}{R}} - R^2 - B^2 \Big)
\end{align*}
\end{proof}

\section{Additional proofs}
\label{sec:proofs}
\begin{replemma}{lemma:offset}
  Let $\r \colon \cX \to [0,1]$ be defined by $\r(x) := \max(h(x) - \eta^{2/3}, 0)$.  The function $\r$ then satisfies
  \begin{equation*}
    S(\r) \leq \eta^{1/2} + 2 \eta^{2/3}.
    \end{equation*}
\end{replemma}

\begin{proof}
  By definition of $S$ and $\r$ we have:
  \begin{align}
    S(r) &= \E[b - \r(x) \ind_{b \geq \r(x)}]
    = \E[b - \r(x)]  + \E[\r(x) \ind_{b < \r(x)}]  \nonumber \\
    & \leq \E[b - h(x)] + \E[h(x) - \r(x)] + P(b < \r(x))  \nonumber \\
    & \leq \eta^{1/2} + \eta^{2/3} + P(h(x) - b > \eta^{2/3}) \label{eq:jensen}\\
    & \leq \eta^{1/2} + \eta^{2/3} + \frac{\E\big[(h(x) - b)^2\big]}{\eta^{4/3}} \label{eq:markov} = 
      \eta^{1/2} + 2 \eta^{2/3} ,
  \end{align}
where \eqref{eq:jensen} is a consequence of Jensen's inequality and
we used Markov's inequality in \eqref{eq:markov}. 
\end{proof}

\begin{repcorollary}{coro:additive}
The following bound holds for any distribution F:
\begin{equation*}
  S \leq  (3 R)^{1/3} \sigma^{2/3} \leq (3 B)^{1/3} \sigma^{2/3} 
\end{equation*}
\end{repcorollary}

\begin{proof}
By Theorem~\ref{THM:EXPBOUND} and using a third order Taylor expansion
we have:
\begin{align*}
  \sigma^2 & \geq 2 R^2 e^{\frac{S}{R}} - R^2 - B^2 \\
&\geq  2 R^2 \Big( 1 + \frac{S}{R}  
+ \frac{S^2}{2 R^2} + \frac{S^3}{6 R^3}\Big)
 - B^2 - R^2 \\
&= 2 R^2 + 2 R S + S^2  + \frac{S^3}{3 R} - B^2 - R^2 \\
&= (S - R)^2 - B^2 +  \frac{S^3}{3 R} = \frac{S^3}{3 R}.
\end{align*}
The proof follows by rearranging terms. 
\end{proof}

\begin{reptheorem}{thm:multapprox}
For any distribution $F$ with mean $B$ and variance $\sigma^2$, the maximum revenue with monopoly reserves, $R$, satisfies: 
\begin{equation*}
\frac{B}{R} \leq 4.78 + 2 \log \big (1 + \frac{\sigma^2}{B^2}\big)
\end{equation*}
\end{reptheorem}
\begin{proof}
Let $\alpha = \frac{B}{R}$. Note that $\alpha \geq 1$. 
We begin by dividing both sides of  the statement of \ref{THM:EXPBOUND} by $R^2$:
\begin{equation*}
  \frac{\sigma^2}{B^2} \alpha^2 + \alpha^2  - 2 e^{\alpha - 1}  \geq - 1 
\end{equation*}
Rearranging, we have:
\begin{equation}
\label{eqn:t}
 \frac{\sigma^2}{B^2} \alpha^2 + \alpha^2    \geq  2 e^{\alpha - 1} - 1 .
 \end{equation}
Since $\alpha \geq 1$, $e^{\alpha  - 1} \geq 1$. Therefore, if Equation \ref{eqn:t} holds, then:
\begin{align*}
& \frac{\sigma^2}{B^2} \alpha^2 + \alpha^2    \geq  e^{\alpha - 1} \\
\Leftrightarrow \quad & \alpha \sqrt{ 1 + \frac{\sigma^2}{B^2}} \geq
                   e^{\frac{\alpha-1}{2}}  \\
\Leftrightarrow \quad & 2\sqrt{e} \sqrt{1 + \frac{\sigma^2}{B^2}} \geq \frac{e^{\alpha / 2}}{\alpha/2}
\end{align*}
Suppose $e^x / x \leq t$ for some fixed $t \geq 2 \sqrt{e}$. Note that
the function $e^x / x$ is increasing in $x$ for $x \geq 1$. Moreover,
at $x = 2 \log t > 1$ we have $e^x / x = t^2 / (2 \log t) \geq t$,
since $t > 2 \log t$ for $t > 2$. Therefore $ x \leq 2 \log t$.

In our situation, we can then conclude that 
$$\alpha \leq 4 \log \left(2\sqrt{e} \sqrt{1 + \frac{\sigma^2}{B^2}}\right) < 4.78 + 2 \log \left(1 + \frac{\sigma^2}{B^2}\right).$$
\end{proof}

\begin{replemma}{lemma:approxcluster}
  Let $h$ be a function such that $\frac{1}{m} \sum_{i=1}^m (h(x_i) - b_i)^2 \leq
  \h \eta^2$, and let $\C^*$ denote the partition that minimizes
  $\Phi(\C)$. If $\C^h$ minimizes  $\Phi_h(\C)$ then
\begin{equation*}
  \Phi(\C^h) \leq \Phi(\C^*)  + 4 m \h \eta. 
\end{equation*}
\end{replemma}
\begin{proof}
  From definition of $\Phi(\C)$ and a straightforward application of
  the triangle inequality we have:
  \begin{flalign*}
    &\Phi(\C^h)
    = \sum_{j=1}^k \sqrt{\sum_{i,i': x_i, x_{i'} \in C^h_j} (b_i - b_{i'})^2}& \\
    & \leq \sum_{j=1}^k  \sqrt{\sum_{i,i' : x_i, x_{i'} \in C^h_j} (h(x_i) - h(x_{i'}))^2}
      +  \sqrt{\sum_{i,i': x_i, x_{i'} \in C^h_j} (h(x_i) - b_i)^2 }
      +  \sqrt{\sum_{i,i': x_i, x_{i'} \in C^h_j} (h(x_{i'}) - b_{i'})^2} & \\
    & = \Phi_h(\C^h)
    + 2 \sum_{j=1}^k \sqrt{m_j \sum_{i : x_i \in C^h_j} (h(x_i) - b_i)^2} \\
    & \leq \Phi_h(\C^h)
    + 2 \sqrt{m} \sqrt{\sum_{j=1}^k \sum_{i: x_i \in C_j^h} (h(x_i) - b_i)^2},
  \end{flalign*}
where we have used Cauchy-Schwarz inequality for the last line.
Using the property of $h$ we can further bound the above expression as
\begin{align}
  \Phi(\C^h) &\leq \Phi_h (\C^h) + 2  m \h \eta \nonumber\\
  & \leq \Phi_h(\C^*)  + 2 m \h \eta \label{eq:phibyphih},
\end{align}
where we have used the fact that $\C^h$ minimizes $\Phi_h$. Proceeding
in the same manner as before, it is easy to see that 
$\Phi_h(\C^*) \leq \Phi(\C^*) + 2 m \h \eta$.
Replacing this bound in \eqref{eq:phibyphih} we recover the statement
of the lemma.
  \end{proof}

\section{Dynamic Program}
\label{sec:algorithm}

\begin{lemma}
Let $y_1 \leq \ldots \leq y_m$, there exists an algorithm with time
complexity in $O(km^2)$ that finds a set of indices $i_0 = 1 \leq
i_1\leq \ldots i_{k-1} \leq i_k = m$ minimizing
\begin{equation*}
  \Phi(i_0, \ldots, i_k) = \sum_{j=1}^l \sqrt{\sum_{i,i' = i_{j-1}}^{i_j} (y_i - y_i')^2}
\end{equation*}
\end{lemma}

\begin{proof}
For every $l \leq k$  and $r \leq m$ define
\begin{align*}
  A_{l,r} &= \min_{i_0 =1 \leq  \ldots \leq  i_l = r}  \sum_{j=1}^k
  \sqrt{\sum_{i,i' = i_{j-1}}^{i_j} (y_i - y_i')^2} 
\end{align*}
It is clear that $A_{k,m}$ is equal to the minimum of $\Phi$. We now
show that $A_{l,r}$ satisfies the recursion:
\begin{equation*}
A_{l,r} = \min_{r' < r}   A_{l - 1, r'} + \sqrt{\sum_{i, i'=r' +1}^r (y_i - y_{i'})^2}.
\end{equation*}
Let $i_0 \leq \ldots \leq i_l$ denote the set of indices defining
$A_{l,r}$ notice that by definition
\begin{equation*}
  A_{l-1, i_{l-1}} \leq \sum_{j=1}^{l-1} \sqrt{\sum_{i,i' = i_{j=1}}^{i_j} (y_i - y_{i'})^2}.
\end{equation*}
Therefore 
\begin{align*}
\min_{r' < r}   A_{l - 1, r'} + \sqrt{\sum_{i, i'=r' +1}^r (y_i -
  y_{i'})^2} & \leq A_{l -1 , i_{l-1}} + \sqrt{\sum_{i,i' = i_{l-1}}^r (y_i -
  y_{i}')^2} \\
& \leq A_{l,r}.
\end{align*}
The reverse inequality is trivial. We have thus shown that 
calculating $A_{k,m}$ requires finding all values $A_{l,r}$ and each
value can be calculated in $O(m)$ time. Therefore, the complexity of
the algorithm is in $O(km^2)$.
\end{proof}

\section{Lower Bounds}
\label{app:lowerbound}

\begin{lemma}
\label{lem:lowerbound}
For any $\epsilon > 0$ there exists a distribution $F$ such that 
$B - R \geq (3 B)^{1/3} \sigma^{2/3} - \epsilon$.
\end{lemma}
\begin{proof}
Let $R = 1$ and consider a distribution $G$ given by $G(x) = 1$ for
$x < 1$ and $G(x) = \frac{1}{x}$ for $x \in [R, M]$. We then have that
the optimal revenue is given by $1$ and the mean of this distribution
is
\begin{equation*}
  B = \int_0^M G(x) = 1 + \log M.
\end{equation*}
On the other hand the second moment of the distribution is given by
\begin{equation*}
 2 \int_0^M x G(x) = 1 + 2 \Big(1 - \frac{1}{M}\Big).
\end{equation*}
Therefore the variance of the distribution is given by:
\begin{equation*}
 \sigma^2 =  2 \Big( 1 - \frac{1}{M} - \log M - \frac{\log M^2}{2}\Big)
\end{equation*}
Using the fact that $\frac{1}{M} = e^{-\log M}$ and using the Taylor
expansion of $x \mapsto e^x$ around 0 we see that the variance is
roughly $\frac{\log^3 M}{3} + o(\log^3 M) $as $M \to 1$. Since $B -
R^* = \log M $ it follows that for any $\epsilon > 0$ there exists $M$
close to $1$ such that  $B - R^* \geq (3 R^*)^{1/3} \sigma^{2/3} - \epsilon$.
\end{proof}

\section{Complexity Bounds}
\label{app:generalization}

\begin{reptheorem}{THM:ALG-GENERALIZATION}
The growth function of the class $G(h,k)$ can be bounded as:
\begin{equation*}
  \Pi(G(h, k), m) \leq \frac{m^{2 k}}{k^k}.
  \end{equation*}
\end{reptheorem}
\begin{proof}
  Let $\cS' = \big((x_1, z_1), \ldots, (x_m, z_m)\big)$ denote a
  sample.Let $\mathcal{G} = \{ \big(\sgn(g(x_1)
  -z_1), \ldots, \sgn(g(x_m)-z_m)\big) | g \in G(h, k) \}$. We proceed
  to bound the cardinality of $\mathcal{G}$.  Notice that a partition
  $\t \in \cT_k$ can divide the set of predictions $h(x_1), \ldots,
  h(x_m)$ into at most $m^{k-1}$ different ways. Indeed, this is
  immediate as a $k$-partition of $[0,1]$ is defined by $k-1$ points
  $t_1, \ldots, t_{k-1}$ and each $t_i$ has at most $m$ distinct
  places to be placed.  Now, fix a partition $\t \in \cT_k$ and let
  $I_j = [t_{j-1}, t_j]$. Let (possibly after relabeling)
  $h(x_1), \ldots, h(x_{m_j})$ denote the points that fall in interval
  $I_j$. Notice that all points falling in $I_j$ share the same
  reserve price $r_j$ thus the number of \emph{labelings} that can be
  obtained in interval $j$ are equal to
  \begin{equation*}
    \big|\{(\sgn(r_j - z_1), \ldots, \sgn(r_j - z_{m_j})) | r_j \in \Rset\}\big| = m_j
    \end{equation*}
  Therefore for a  fixed partition there are at most
  $\prod_{j=1}^k m_j \leq  \Big(\frac{m}{k} \Big)^k$ labelings. Since
  there are at most $m^{k-1}$ partitions then we must have:
  \begin{equation*}
    \Pi(G(h, k)) \leq \frac{m^{2 k - 1}}{k^k}
    \end{equation*}
  \end{proof}

\newcommand{\ee}{\boldsymbol{\eps}}
\ignore{
\begin{replemma}{lemma:complex}
  Let $G$ be a class of functions such that for every $m$ and every
  choice of $\big( (x_1, b_1), \ldots, (x_m, b_m)\big)$ there exists
  $r \in G$ such that $r(x_i) = b_i$. Then, for every $m$ we have:
  \begin{equation*}
    \Pi(G, m) = 2^m.
  \end{equation*}
\begin{proof}
Let $ \big((x_1, z_1), \ldots, (x_m, z_m) \in \cX \times \Rset$ and
$\eta > 0$. Given $\boldsymbol {\eps} \in \set{-1,1}^m$ define
$r_{\ee} \in G$ such that $r(x_i) = z_i + \eta$ if $\ee_i =1$ and
$r_{\ee}(x_i) = z_i - \eta$ if $\ee_i = -1$ notice that $r_{\ee}$ is
well defined by definition of $G$. Finally, let $s_{\ee}
= \big(\sgn(r_{\ee}(x_1) - z_1), \ldots, \sgn(r_{\ee}(x_m) -
z_m)\big)$. It is clear that the mapping
$\phi \colon \ee \mapsto s_{\ee}$ is inyective therefore
$|phi(\set{-1,1}^m| \geq |\set{-1,1}^| = 2^m$. But by definition
$\Pi(G, m) \geq |\phi(\set{-1,1}^m|$.
\end{proof}
\end{replemma}
}